\documentclass[conference]{IEEEtran}

\usepackage{cite}
\usepackage{amsmath,amssymb,amsfonts}
\usepackage{graphicx}
\usepackage{textcomp}
\usepackage{xcolor}
\usepackage{booktabs}
\usepackage{multirow}
\usepackage{colortbl}
\usepackage{hyperref}
\usepackage{algorithm}
\usepackage{algorithmic}
\usepackage[normalem]{ulem}
\useunder{\uline}{\ul}{}

\hypersetup{
    colorlinks=true,
    linkcolor=blue,
    citecolor=blue,
    urlcolor=blue,
}

\newcommand{\cmmnt}[1]{}

\newtheorem{theorem}{Theorem}
\newtheorem{proposition}[theorem]{Proposition}

\begin{document}

\title{Context-aware Modality-Topology Co-Alignment for Multimodal Attributed Graphs}

\author{\IEEEauthorblockN{\textbf{Sirui Zhang\quad  Xu Wang\quad  Zhengyu Wu\quad  Xunkai Li\quad  Hongchao Qin*}}}

\maketitle

\begin{abstract}
Multimodal Attributed Graphs (MAGs) model real-world entities by coupling graph topology with heterogeneous semantic attributes such as text and images.
This data structure naturally supports both graph-centric tasks, which require structural and class-discriminative representations, and modality-centric tasks, which require fine-grained cross-modal correspondence.
However, existing MAG methods often rely on a fixed graph context or a uniformly fused multimodal representation.
Such designs can lead to task-agnostic context propagation and over-compressed cross-modal fusion, making it difficult to satisfy different task requirements while preserving modality-specific evidence.
To address this challenge, we propose \underline{\textbf{CoMAG}} (Context-aware Modality-Topology Co-Alignment), a unified MAG backbone that learns task-adaptive reliable contexts and performs modality-preserving alignment within those contexts.
Concretely, CoMAG first performs Reliable Context Learning by estimating edge reliability from multimodal semantic consistency, supplementing the raw topology with semantic neighbors, and selecting context components through a task-aware gate.
It then performs Modality-preserving Hop-token Alignment by maintaining modality-specific multi-hop trajectories, matching modality-hop tokens across modalities, and using shared-private representation decoupling to produce graph representations and modality representations from the same forward pass.
We further provide theoretical analysis of stable propagation, over-smoothing mitigation, and modality-collapse control.
Experiments on nine OpenMAG datasets compare CoMAG with feature-only, graph-only, multimodal, and unified MAG baselines across graph-level prediction, modality matching, and graph-conditioned generation protocols.
The results show that CoMAG achieves the best reported performance among compared methods, indicating that task-adaptive reliable contexts and modality-preserving alignment jointly strengthen structural prediction, cross-modal matching, and graph-conditioned generation while retaining sparse edge-linear complexity. Code is available at \url{https://anonymous.4open.science/r/CoMAG}.
\end{abstract}


\section{Introduction}

Multimodal Attributed Graphs (MAGs) provide a natural abstraction for real-world entities whose relational dependencies and semantic attributes are inseparable.
In e-commerce\cite{e_commerce1}, products are connected by co-purchase relations while being described by titles, reviews, and images.
In cellular networks\cite{cellular_1}, cell infrastructures can be modeled as connected nodes, which are described with heterogeneous attributes such as radio signals, and geographic context.
MAG learning therefore aims to jointly exploit graph topology and multimodal semantics encoded by modern vision, language, or vision-language models~\cite{wei2019mmgcn,tao2022mgat}.
This setting supports two complementary task families.
\textit{Graph-centric tasks}, such as node classification, link prediction, and node clustering, evaluate whether representations preserve structural, relational, and class-discriminative information.
\textit{Modality-centric tasks}, such as cross-modal retrieval, modality matching, modality alignment, and graph-conditioned generation, evaluate whether representations capture fine-grained semantic correspondences across modalities.
A unified MAG backbone is therefore expected to support graph reasoning and multimodal semantic matching from the same representation learning pipeline.

Despite notable progress, existing MAG methods remain limited by the assumption that a single graph context or a single fused representation can serve all downstream objectives.
\textbf{(1) Task-agnostic context propagation.}
Graph-centric and modality-centric tasks often require different neighborhoods.
Classification may prefer label-consistent neighbors, link prediction may rely on structural complementarity, clustering may require compact yet separable communities, and retrieval may benefit from semantically aligned nodes that are not directly connected in the original graph.
However, many methods propagate over the original topology or a uniformly learned topology, which can amplify noisy edges, overlook missing semantic relations, and treat task-conflicting edges in the same way.
\textbf{(2) Over-compressed cross-modal fusion.}
Existing fusion or alignment mechanisms often encourage different modalities to collapse into a common representation space.
While such compression can highlight shared semantics, it may also remove modality-specific details that are essential for retrieval, matching, generation, and fine-grained alignment.
Recent studies on neighborhood tokenization, modality-aware propagation, task-specific topology, and shared-private disentanglement provide useful components~\cite{tang2024graphgpto,lai2024ntsformer,zhang2024miggt,shi2024campa,liu2024tmte,yang2022dgf,cheng2023dmgc}, but they do not fully resolve how to learn task-adaptive graph contexts while preserving modality-specific evidence.

Therefore, we center this research around the following question: \textit{How can a MAG backbone learn reliable contexts for different tasks while aligning modalities without erasing their distinctive information?}
Our starting point is that context construction and modality alignment should be treated as two coupled but distinct problems.
Rather than directly aggregating all neighbors from the raw graph, the model should identify which structural edges are reliable, recover semantic neighbor relations that the original graph misses, and adapt the resulting context to the target task.
Rather than forcing all modalities into a single fused embedding, the model should maintain modality-specific propagation trajectories and align them at a finer granularity, so that cross-modal consensus and modality-private evidence can both be retained.
This perspective shifts MAG learning from simple topology-modality fusion toward context-aware modality-topology co-alignment.

To this end, we propose \textbf{CoMAG} (\textbf{Co}ntext-aware \textbf{M}odality-topology co-\textbf{A}lignment for multimodal attributed \textbf{G}raphs), a unified framework designed to serve graph-centric and modality-centric tasks through a shared backbone.
CoMAG is built on two technical pillars.
\textit{Reliable Context Learning} estimates edge reliability from multimodal semantic consistency, supplements the original graph with semantic neighbors, and selects an appropriate context according to the task.
\textit{Modality-preserving Hop-token Alignment} propagates each modality along the learned context as a multi-hop trajectory, treats each modality-hop pair as an alignment token, and decouples the resulting representation into shared consensus and modality-specific residual components.
The shared component supports graph-oriented reasoning, while the modality-specific components preserve the discriminative information required by retrieval, matching, and generation.
Under the OpenMAG protocol, these design choices yield the best reported performance among evaluated methods on graph-level prediction and modality-level matching and generation, suggesting that reliable topology, semantic context recovery, and modality-private evidence are complementary for unified MAG learning.
Our main contributions are summarized below.
\begin{enumerate}
    \item \textbf{Problem Reframing.} We identify unified MAG representation learning as a problem of task-adaptive context construction and modality-preserving alignment, rather than a direct extension of graph aggregation or multimodal fusion.
    \item \textbf{Novel Framework.} We introduce CoMAG, which integrates reliable context learning, modality-specific hop trajectories, hop-token cross-modal matching, and shared-private decoupling into a unified MAG backbone.
    \item \textbf{SOTA Performance.} We validate CoMAG through extensive evaluations to demonstrate its superior performance against competitive baselines and exhibit its robustness against various challenging scenarios.
\end{enumerate}

\section{Related Work}

\subsection{Multimodal Attributed Graph Learning}

Multimodal attributed graph learning considers relational data whose nodes carry heterogeneous semantic evidence.
Recent surveys cast this setting as a general multimodal graph-learning problem rather than a recommendation-only formulation~\cite{ektefaie2023multimodal,peng2024mmglsurvey}.
New benchmarks further broaden MAG evaluation across heterogeneous modalities, graph-centric prediction, and modality-centric reasoning~\cite{mu2024openmag,zhu2025mmgraph,yan2024magb}.
Representative models inject modality-aware convolution, attention, and recommendation objectives into message passing~\cite{wei2019mmgcn,tao2022mgat,jia2023mhgat,LGMRec}.
Other studies target clustering, biomedical analysis, or architecture search over multimodal graph structures~\cite{yang2022dgf,cheng2023dmgc,mgnet,DBLP:conf/aaai/CaiWLZ024}.
Recent graph-transformer and unified-embedding methods further bind multimodal graph signals in shared spaces~\cite{he2025unigraph2,zhang2024miggt,lai2024ntsformer}.
Foundation-model and generative directions extend this line to graph-enhanced multimodal understanding, graph-language assistants, and image generation from MAGs~\cite{tang2024graphgpto,li2024graph4mm,MMGL}.
Together, these works confirm the need to model topology and modality semantics jointly, but most still rely on task-specific architectures or a fixed graph context.
CoMAG retains the broad MAG objective while learning task-adaptive reliable contexts and modality-preserving outputs within one forward pipeline.

\subsection{Reliable Graph Context and Multi-hop Propagation}

Graph neural networks supply the standard machinery for propagating information over relational structures.
Classical message-passing models aggregate local neighborhoods through convolution, attention, sampling, or higher-order aggregation~\cite{kipf2017gcn,defferrard2016chebnet,velickovic2018gat,hamilton2017graphsage,xu2018powerful}.
Later analyses identify over-smoothing and depth limits, while residual and reversible designs improve deep propagation~\cite{li2018deeper,oono2020deep,chen2020gcnii,revgat}.
Propagation-filter methods further control how signals accumulate across multiple hops~\cite{klicpera2019appnp,wu2019sgc,chien2021gprgnn}.
Heterophily studies and graph transformer variants show that useful context can differ sharply from local homophilous neighborhoods~\cite{nguyen2022u2gnn,lim2021linkx,luan2022heterophily,platonov2023heterophily}.
In MAGs, context reliability is even more task-dependent because an edge should be trusted only when its structural relation is compatible with multimodal semantics.
Task-specific topology modification and multimodal graph transformers indicate that context should adapt to the objective~\cite{liu2024tmte,lai2024ntsformer,zhang2024miggt}.
CoMAG follows this principle by constructing a reliability-gated structural graph, a semantic complement graph, and a task-conditioned context mixture before propagation, making context learning part of the backbone itself.

\subsection{Cross-modal Alignment and Shared-private Representation Learning}

Cross-modal alignment seeks comparable representations for heterogeneous modalities while retaining evidence that is unique to each source.
Vision-language pretraining and multimodal binding models provide strong alignment priors~\cite{radford2021clip,girdhar2023imagebind,chen2020simple}, while modern text and vision encoders supply expressive node attributes for MAGs~\cite{dosovitskiy2021vit,oquab2024dinov2,sbert,raffel2020t5,roberta}.
Graph-aware alignment methods inject relational evidence into image-text matching, multimodal path alignment, and graph-enhanced multimodal understanding~\cite{gsmn,shi2024campa,li2024graph4mm,tang2024graphgpto}.
Graph contrastive pretraining similarly learns robust representations by contrasting contextual signals~\cite{GraphCL,GCC,velickovic2018dgi,peng2020gmi}, and masked graph modeling reconstructs hidden evidence from surrounding context~\cite{graphmae2,li2023MaskGAE}.
Yet forcing all modalities into one fused embedding can obscure modality-specific cues, a concern also studied in shared-specific and missing-modality learning~\cite{wang2023shaspec,ma2021smil,ma2022missingtransformer,wu2024missingmodalitysurvey}.
This issue is especially visible in MAGs, where cross-modal agreement may emerge at different propagation depths.
A textual neighborhood can become category-discriminative after several hops, whereas visual evidence may remain local and appearance-driven.
Final-embedding alignment alone cannot tell whether agreement comes from shared semantics or from useful modality-private evidence that should be preserved.
Consequently, a unified MAG backbone needs alignment that is fine-grained enough to compare modality-hop evidence while still separating consensus from modality-specific residual information for $z_i^g$ and $e_i^{(m)}$.

\section{Preliminaries}

\subsection{Multimodal Attributed Graph}

A Multimodal Attributed Graph (MAG) is denoted as $\mathcal{G}=(\mathcal{V},\mathcal{E},\{X^{(m)}\}_{m=1}^{M})$, where $\mathcal{V}$ contains $N=|\mathcal{V}|$ nodes, $\mathcal{E}\subseteq\mathcal{V}\times\mathcal{V}$ is represented by adjacency matrix $A\in\mathbb{R}^{N\times N}$, and $X^{(m)}\in\mathbb{R}^{N\times d_m}$ is the node feature matrix for modality $m$.
The $i$-th node has modality observations $\{x_i^{(m)}\}_{m=1}^{M}$, where $x_i^{(m)}$ is the $i$-th row of $X^{(m)}$.
Different modalities may have different raw dimensions, encoders, and noise patterns, so a MAG learner must combine relational and modality evidence without assuming equal reliability for every task.
Given a task index $\tau$, a unified backbone should produce $z_i^g\in\mathbb{R}^{d}$ for graph-centric objectives and $e_i^{(m)}\in\mathbb{R}^{d}$ for modality-centric objectives.
We write $Z^g=[z_1^g,\ldots,z_N^g]^\top$ and $E^{(m)}=[e_1^{(m)},\ldots,e_N^{(m)}]^\top$.
Our experiments use text and visual features, while the notation allows arbitrary $M$.

\subsection{Learning Objectives}
\label{subsec:learning_objectives}

CoMAG is trained to support graph-centric and modality-centric objectives within one representation-learning problem.
For graph-centric supervision, $z_i^g$ supports node classification, link prediction, and node clustering.
We summarize these terms as $\mathcal{L}_{\mathrm{graph}}=\mathcal{L}_{\mathrm{cls}}+\mathcal{L}_{\mathrm{lp}}+\mathcal{L}_{\mathrm{cluster}}$, where $\mathcal{L}_{\mathrm{cls}}$ predicts node labels, $\mathcal{L}_{\mathrm{lp}}$ scores whether a pair $(i,j)$ should be connected, and $\mathcal{L}_{\mathrm{cluster}}$ encourages structurally and semantically related nodes to form coherent groups.
This objective family requires $Z^g$ to retain topology, neighborhood context, and class-level semantics.

For modality-centric supervision, $\{e_i^{(m)}\}_{m=1}^{M}$ should make different modalities comparable without erasing modality-specific evidence.
For $m\neq n$, the pair $(e_i^{(m)},e_i^{(n)})$ is positive because both representations describe node $i$, while $(e_i^{(m)},e_j^{(n)})$ with $j\neq i$ gives a negative pair.
The retrieval or matching objective therefore increases similarity for aligned modality pairs and decreases it for mismatched pairs, and the full training loss instantiates this behavior through the modality term in Section~\ref{subsec:optimization}.
Generation and alignment tasks reuse the same modality-aware representations as graph-conditioned semantic evidence.

Two auxiliary objectives regularize the shared backbone.
The orthogonality objective $\mathcal{L}_{\perp}$ separates cross-modal consensus from modality-private residuals, preventing alignment from collapsing all modalities into one over-compressed embedding.
The smoothness objective $\mathcal{L}_{\mathrm{smooth}}$ regularizes $Z^g$ over the learned context graph so reliable neighbors have compatible graph representations.
Together, CoMAG optimizes $\mathcal{L}=\lambda_g\mathcal{L}_{\mathrm{graph}}+\lambda_m\mathcal{L}_{\mathrm{ret}}+\lambda_{\perp}\mathcal{L}_{\perp}+\lambda_s\mathcal{L}_{\mathrm{smooth}}$, with detailed definitions in Section~\ref{subsec:optimization}.
Thus, the preliminary objective is to learn graph and modality representations jointly without sacrificing structural discrimination or modality-specific evidence.

\begin{figure*}[t]
    \centering
    \includegraphics[width=\linewidth]{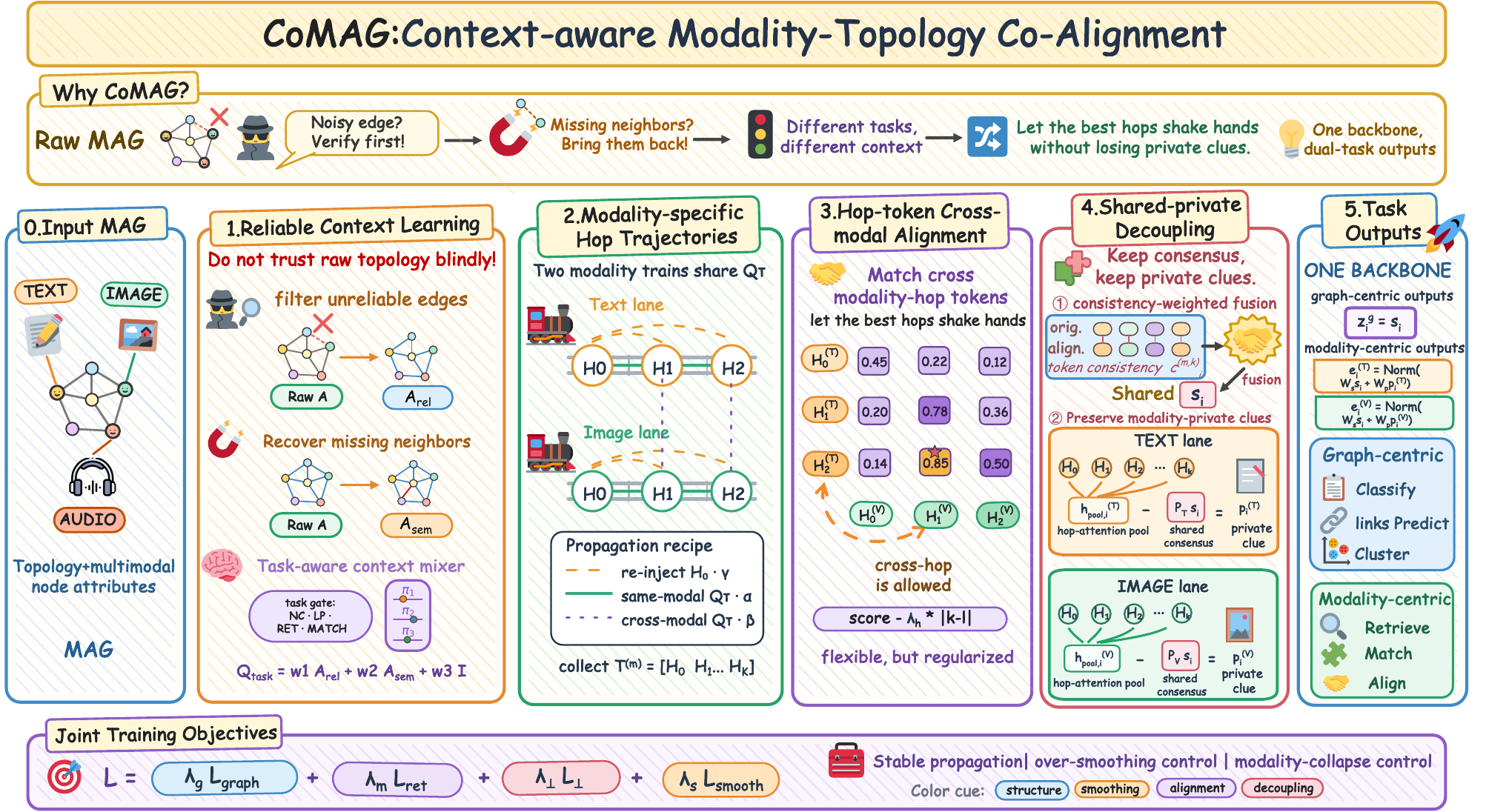}
    \caption{\textbf{CoMAG architecture.}
    CoMAG builds reliable task-adaptive contexts, propagates modality-specific hop trajectories, aligns hop tokens, and decouples shared and private representations for graph- and modality-level outputs.}
    \label{fig:architecture}
\end{figure*}

\section{Methodology}

\subsection{Overview of CoMAG}
\label{subsec:comag_overview}

CoMAG addresses the two limitations identified earlier, namely fixed context propagation and over-compressed modality fusion.
Given a MAG with $M$ modalities, it first learns a task-adaptive context graph, then propagates each modality as a separate multi-hop trajectory, aligns modality-hop evidence across modalities, and finally separates shared graph evidence from modality-private residual evidence.
Graph construction, propagation, alignment, and decoupling remain distinct stages, so each stage addresses one failure mode while the full pipeline is optimized end to end.
Algorithm~\ref{alg:comag_forward} summarizes the resulting inference and training procedure at the module level.

\subsection{Reliable Context Learning}
\label{subsec:reliable_context}

Reliable context learning constructs the graph along which information will propagate.
Instead of treating every observed edge as equally useful, CoMAG asks whether the edge is supported by the modalities attached to its endpoints.
Each modality is first mapped into the same hidden dimension, making textual and visual evidence comparable while preserving modality-specific noise patterns.

\noindent\textbf{Edge reliability estimation.}
For each observed edge $(i,j)\in\mathcal{E}$, CoMAG compares whether same-modality and cross-modality evidence agree across the two endpoints.
Agreement suggests that the edge reflects a meaningful relation, while conflict tells the model to reduce the edge influence before message passing.
This reliability decision is parameterized as
\begin{equation}
\begin{aligned}
    R_{ij}
    =
    \sigma\!\Big(
    g_R\big(
    [A_{ij}, S_{ij}^{\mathrm{intra}},
    S_{ij}^{\mathrm{cross}},
    |S_{ij}^{\mathrm{intra}}-S_{ij}^{\mathrm{cross}}|]
    \big)
    \Big),
\end{aligned}
\label{eq:edge_reliability}
\end{equation}
where $S_{ij}^{\mathrm{intra}}$ and $S_{ij}^{\mathrm{cross}}$ denote within-modality and cross-modality consistency, and $g_R(\cdot)$ is a lightweight scoring network.
The resulting value $R_{ij}\in[0,1]$ determines how strongly the observed edge should participate in propagation.

\noindent\textbf{Semantic context recovery.}
Filtering unreliable edges improves the observed topology but does not restore relations that are missing from the graph.
CoMAG therefore builds a semantic complement graph from cross-modal similarity and keeps only high-confidence neighbors.
The observed topology protects message passing from spurious edges that already exist, while the semantic complement supplies context that the raw topology fails to expose.
The two context components are
\begin{equation}
\begin{aligned}
    A^{\mathrm{rel}}
    &=
    \operatorname{Norm}(A \odot R),
    \\
    A^{\mathrm{sem}}
    &=
    \operatorname{Norm}
    \left(
    \operatorname{TopK}
    \left(\bar{S}^{\mathrm{cross}}\right)
    \right).
\end{aligned}
\label{eq:context_components}
\end{equation}
Here $\bar{S}^{\mathrm{cross}}$ denotes the cross-modal semantic similarity matrix used to recover high-confidence missing neighbors.

\noindent\textbf{Task-adaptive context gate.}
Different downstream objectives do not always need the same structural prior.
CoMAG uses the task embedding and a graph-level summary to choose how much mass to assign to reliable topology, semantic complement, and self information.
The resulting context graph is
\begin{equation}
    Q_\tau
    =
    \pi_{\tau,1} A^{\mathrm{rel}}
    +
    \pi_{\tau,2} A^{\mathrm{sem}}
    +
    \pi_{\tau,3} I .
    \label{eq:Q_tau}
\end{equation}
Here $A^{\mathrm{rel}}$ is the normalized reliability-filtered topology, $A^{\mathrm{sem}}$ is the normalized semantic complement graph, and $\pi_\tau$ is a task-conditioned simplex weight.
The identity channel keeps self evidence available, which prevents the context graph from forcing every node to rely only on neighbors.

\subsection{Modality-specific Multi-hop Context Trajectory}
\label{subsec:stage1}

Once $Q_\tau$ is built, CoMAG propagates each modality through its own trajectory.
Every modality receives the same task-aware context graph while keeping its own hidden state across hops, allowing each modality to react differently to the same graph context.

\noindent\textbf{Task-conditioned propagation coefficients.}
For task $\tau$, CoMAG generates modality-specific coefficients by
\begin{equation}
    (\gamma_{\tau,m}, \alpha_{\tau,m}, \beta_{\tau,m})
    =
    \operatorname{softmax}
    \left(W_c^{(m)} q_\tau\right),
    \label{eq:hop_coeff}
\end{equation}
so each modality can adjust how much it relies on its original signal, its own propagated context, and contextual evidence from other modalities.
The softmax form yields a convex update, which is the condition used later in Theorem~\ref{thm:stability}.

\noindent\textbf{Multi-hop context propagation.}
Starting from $H_{\tau,0}^{(m)}=H_0^{(m)}$, CoMAG updates modality $m$ as
\begin{equation}
\begin{aligned}
    \bar{H}_{\tau,k}^{(-m)}
    &=
    \frac{1}{M-1}
    \sum_{n\neq m}
    H_{\tau,k}^{(n)},
    \\
    H_{\tau,k+1}^{(m)}
    &=
    \gamma_m H_0^{(m)}
    +
    \alpha_m Q_\tau H_{\tau,k}^{(m)}
    +
    \beta_m Q_\tau \bar{H}_{\tau,k}^{(-m)} .
\end{aligned}
\label{eq:hop_propagation}
\end{equation}
The first term re-injects the original modality evidence, the second propagates same-modality context, and the third lets other modalities correct the trajectory through the learned graph.
Each hop also uses normalization and non-linearity in implementation.
CoMAG keeps the whole sequence of hidden states after $K$ steps, so early states preserve local modality evidence while later states encode broader graph context.

\subsection{Hop-token Cross-modal Alignment}
\label{subsec:hop_token_alignment}

Cross-modal correspondence need not be hop-synchronous.
CoMAG treats each modality-hop state as an alignment token and attaches modality and hop-position embeddings, so matching remains aware of both evidence source and receptive-field depth.
The token sequence is
\begin{equation}
\begin{aligned}
    u_i^{(m,k)}
    &=
    H_{\tau,k}^{(m)}[i] + r_m + \ell_k,
    \\
    U_i^{(m)}
    &=
    \left[
    u_i^{(m,0)}, \ldots, u_i^{(m,K)}
    \right],
\end{aligned}
\label{eq:hop_token}
\end{equation}
where $r_m$ and $\ell_k$ are learnable modality and hop-position embeddings.

\noindent\textbf{Distance-penalized cross-modal attention.}
For each ordered modality pair $(m,n)$, CoMAG compares the token sequence of modality $m$ against that of modality $n$.
The attention score may match across hops, but the distance penalty prevents remote hops from dominating without strong semantic support.
The matching matrix is
\begin{equation}
\begin{aligned}
    B_i^{m \leftarrow n}
    =
    \operatorname{softmax}
    \left(
    \frac{Q_i^{(m)}(K_i^{(n)})^\top}{\sqrt{d_h}}
    -
    \lambda_h D_{\mathrm{hop}}
    \right),
\end{aligned}
\label{eq:hop_matching}
\end{equation}
where $D_{\mathrm{hop}}[k,l]=|k-l|$, and $Q_i^{(m)}$ and $K_i^{(n)}$ are projected token sequences.
The coefficient $\lambda_h$ encourages nearby-hop matching when semantic scores are comparable while still allowing cross-hop transfer when the evidence justifies it.

\noindent\textbf{Aligned token construction.}
The matching matrix imports value tokens from modality $n$ into the hop positions of modality $m$.
CoMAG blends the imported evidence with the original tokens, keeping each modality trajectory visible while adding cross-modal context.
\begin{algorithm}[t]
\caption{CoMAG Training Pipeline}
\label{alg:comag_forward}
\begin{algorithmic}[1]
  \REQUIRE MAG $\mathcal{G}$, task embedding $q_\tau$, hop depth $K$, semantic budget TopK, weights $\lambda_g,\lambda_m,\lambda_\perp,\lambda_s$
  \ENSURE Context graph $Q_\tau$, graph outputs $\{z_i^g\}$, modality outputs $\{e_i^{(m)}\}$, and loss $\mathcal{L}$
  \STATE Encode each raw modality into hidden features $\{H_0^{(m)}\}_{m=1}^{M}$.
  \STATE Estimate edge reliability $R_{ij}$ for observed edges with Eq.~\eqref{eq:edge_reliability}.
  \STATE Build context components by Eq.~\eqref{eq:context_components}.
  \STATE Mix the two graphs with self connections to obtain $Q_\tau$ by Eq.~\eqref{eq:Q_tau}.
  \STATE Generate task-conditioned propagation coefficients by Eq.~\eqref{eq:hop_coeff}.
  \FOR{$k=0,\ldots,K-1$}
    \STATE Update each modality trajectory with Eq.~\eqref{eq:hop_propagation}.
  \ENDFOR
  \STATE Build modality-hop tokens by Eq.~\eqref{eq:hop_token} and align them with Eq.~\eqref{eq:hop_matching}.
  \STATE Fuse reliable consensus tokens into $s_i$ by Eq.~\eqref{eq:shared}.
  \STATE Extract modality-private residuals $p_i^{(m)}$ by Eq.~\eqref{eq:private}.
  \STATE Decode $z_i^g$ and $e_i^{(m)}$ by Eq.~\eqref{eq:final_outputs}.
  \IF{training}
    \STATE Apply the orthogonal constraint in Eq.~\eqref{eq:orth_loss}.
    \STATE Optimize the total objective in Eq.~\eqref{eq:total_loss}.
  \ENDIF
  \STATE \textbf{return} $Q_\tau,\{z_i^g\},\{e_i^{(m)}\}$ and $\mathcal{L}$ if training.
\end{algorithmic}
\end{algorithm}
\subsection{Shared-private Representation Decoupling}
\label{subsec:stage2}

Shared-private decoupling prevents cross-modal alignment from erasing useful modality-specific evidence.
The aligned tokens form a shared consensus, while modality-level tasks retain details that may not be useful for graph-centric prediction.
CoMAG therefore constructs one shared representation and one private residual for each modality.

\noindent\textbf{Shared fusion.}
For each token, CoMAG measures how consistent the original modality-hop state is with its aligned counterpart.
Consistent tokens receive larger weights, while unreliable or weakly aligned tokens contribute less to the shared representation.
The shared representation is
\begin{equation}
    s_i
    =
    \sum_{m=1}^{M}
    \sum_{k=0}^{K}
    a_i^{(m,k)}
    f_i^{(m,k)} .
    \label{eq:shared}
\end{equation}
Here $a_i^{(m,k)}$ is a learned consistency-aware token weight, and $f_i^{(m,k)}$ is the fused token feature built from the original and aligned tokens.
This weighted fusion gives graph-centric tasks a stable consensus representation.

\noindent\textbf{Private residual extraction.}
The private branch pools the original trajectory of each modality and removes the part already explained by the shared representation.
\begin{equation}
    p_i^{(m)}
    =
    h_{\mathrm{pool},i}^{(m)}
    -
    P_m s_i ,
    \label{eq:private}
\end{equation}
where $h_{\mathrm{pool},i}^{(m)}$ is the pooled trajectory feature and $P_m$ is a modality-specific projection.
The subtraction makes the private branch carry residual evidence.

\noindent\textbf{Output representations and orthogonality.}
CoMAG uses $s_i$ as the graph-centric representation $z_i^g$ and combines $s_i$ with $p_i^{(m)}$ to obtain the modality-centric representation $e_i^{(m)}$.
\begin{equation}
\begin{aligned}
    z_i^g
    &=
    s_i,
    \\
    e_i^{(m)}
    &=
    \operatorname{Norm}
    \left(
    W_s s_i + W_p p_i^{(m)}
    \right).
\end{aligned}
\label{eq:final_outputs}
\end{equation}
This output design lets graph prediction rely on cross-modal consensus while retrieval, matching, and generation retain modality-specific cues.
CoMAG further limits redundancy between shared and private subspaces with
\begin{equation}
    \mathcal{L}_{\perp}
    =
    \frac{1}{NM}
    \sum_{i=1}^{N}
    \sum_{m=1}^{M}
    \left(
    \frac{
    s_i^\top p_i^{(m)}
    }{
    \|s_i\|_2 \|p_i^{(m)}\|_2
    }
    \right)^2 .
    \label{eq:orth_loss}
\end{equation}
Proposition~\ref{prop:collapse} explains why reducing this alignment helps preserve modality-private information.

\subsection{Optimization}
\label{subsec:optimization}

CoMAG is trained with graph supervision, modality supervision, and two auxiliary regularizers.
The graph term covers available classification, link prediction, and clustering targets, while the modality term pulls paired node modalities together and separates mismatched pairs.
The smoothness term regularizes $Z^g$ over $Q_\tau$, so compatibility follows reliable context edges rather than the raw graph alone.
The complete objective is
\begin{equation}
    \mathcal{L}
    =
    \lambda_g \mathcal{L}_{\mathrm{graph}}
    +
    \lambda_m \mathcal{L}_{\mathrm{ret}}
    +
    \lambda_{\perp} \mathcal{L}_{\perp}
    +
    \lambda_s \mathcal{L}_{\mathrm{smooth}} .
    \label{eq:total_loss}
\end{equation}
The weights $\lambda_g$, $\lambda_m$, $\lambda_{\perp}$, and $\lambda_s$ balance task supervision, modality alignment, shared-private separation, and context smoothness.
The graph and modality terms train the two output families, while the orthogonal and smoothness terms keep consensus and context reliable without forcing both objectives into the same representation.
\begin{table*}[t]
\centering
\caption{OpenMAG dataset statistics.}
\label{tab:datasets}
\fontsize{9pt}{10.5pt}\selectfont
\renewcommand{\arraystretch}{1.1}
\setlength{\tabcolsep}{10pt}
\begin{tabular*}{\textwidth}{@{\extracolsep{\fill}}lrrclcc}
\toprule
\textbf{Dataset} &
\textbf{Nodes} &
\textbf{Edges} &
\textbf{Labels} &
\textbf{Modalities} &
\textbf{Domain} &
\textbf{Train/Val/Test} \\
\midrule
Movies       & 16,672 & 218,390   & 20 & Text,Visual & E-Commerce           & 60\%/20\%/20\% \\
Grocery      & 17,074 & 171,340   & 20 & Text,Visual & E-Commerce           & 60\%/20\%/20\% \\
Toys         & 20,695 & 126,886   & 18 & Text,Visual & E-Commerce           & 60\%/20\%/20\% \\
DY           & 8,299  & 35,627    & -  & Text,Visual & Video Recommendation & 60\%/20\%/20\% \\
KU           & 5,370  & 22,052    & -  & Text,Visual & Video Recommendation & 60\%/20\%/20\% \\
Bili\_dance  & 2,307  & 9,127     & -  & Text,Visual & Video Recommendation & 60\%/20\%/20\% \\
RedditS      & 15,894 & 566,160   & 20 & Text,Visual & Social Media         & 60\%/20\%/20\% \\
Flickr30k    & 31,783 & 181,151   & -  & Text,Visual & Image Networks       & 60\%/20\%/20\% \\
SemArt       & 21,382 & 1,216,432 & -  & Text,Visual & Art Networks         & 60\%/20\%/20\% \\
\bottomrule
\end{tabular*}
\end{table*}
\section{Theoretical Analysis}

We analyze the linear core of Eq.~(5). The following results provide compact theoretical justification for stable propagation, over-smoothing mitigation, hop-distance regularization, and shared-private collapse control.

\begin{theorem}[Stable Context Propagation]
\label{thm:stability}
Assume $Q_\tau$ is non-negative and row-normalized, and the propagation coefficients in Eq.~\eqref{eq:hop_coeff} satisfy $\gamma_m+\alpha_m+\beta_m=1$ with $\gamma_{\min}=\min_m \gamma_m>0$.
Then the stacked propagation core of Eq.~\eqref{eq:hop_propagation} can be written as
\begin{equation}
    \mathbf{H}_{k+1}=P_\tau\mathbf{H}_{k}+\Gamma\mathbf{H}_{0},
    \label{eq:theory_recurrence}
\end{equation}
where $\rho(P_\tau)\leq 1-\gamma_{\min}<1$.
Therefore, $\{\mathbf{H}_{k}\}$ converges to the unique fixed point
\begin{equation}
    \mathbf{H}^{*}
    =
    (I-P_\tau)^{-1}\Gamma\mathbf{H}_{0},
    \label{eq:theory_fixed_point}
\end{equation}
with asymptotic linear rate at most $1-\gamma_{\min}$.
\end{theorem}

\noindent\textit{Proof sketch.}
Stacking modalities gives a homogeneous operator with row mass at most $1-\gamma_{\min}$, so $\rho(P_\tau)<1$ and the Neumann series yields Eq.~\eqref{eq:theory_fixed_point}.
The result separates adaptivity from instability because the task gate may change the context mixture and the limiting representation, but the residual mass keeps the propagation operator contractive under the stated normalization.

\begin{theorem}[Residual Mitigates Over-smoothing]
\label{thm:oversmoothing}
Let $\Pi$ be the centering projection that removes the constant node-wise component, and let $\tilde{\mathbf{H}}_k=\Pi\mathbf{H}_k$.
For any centered eigencomponent of the stacked propagation matrix $P_\tau$ with eigenvalue $\mu$ satisfying $|\mu|<1$, propagation without residual has $K$-hop response
\begin{equation}
    g_{\mathrm{nores}}(\mu,K)=\mu^{K},
    \label{eq:nores_response}
\end{equation}
which vanishes as $K$ grows.
With residual injection, the corresponding fixed-point response is
\begin{equation}
    g_{\mathrm{res}}(P_\tau)
    =
    (I-P_\tau)^{-1}\Gamma,
    \label{eq:res_response}
\end{equation}
which is finite because $\rho(P_\tau)<1$ and remains non-zero for centered input components retained by $\Gamma$.
\end{theorem}

\noindent\textit{Proof sketch.}
Without residual injection, centered components decay as Eq.~\eqref{eq:nores_response}.
CoMAG re-injects the initial signal, giving the finite geometric filter in Eq.~\eqref{eq:res_response}.
This explains why CoMAG keeps the whole trajectory rather than using only the final propagated state.
Early hops preserve local modality evidence, later hops encode task-conditioned context, and the residual channel prevents the trajectory from degenerating into a purely low-frequency signal.

\begin{proposition}[Hop-distance Penalty Regularizes Cross-hop Alignment]
\label{prop:hop_distance}
For a fixed query hop $k$ in Eq.~\eqref{eq:hop_matching}, let $a_{kl}$ denote the scaled semantic attention score before applying the hop-distance penalty.
For two candidate hops $l_1$ and $l_2$, the attention ratio satisfies
\begin{equation}
    \frac{
    B_i^{m\leftarrow n}[k,l_2]
    }{
    B_i^{m\leftarrow n}[k,l_1]
    }
    =
    \exp
    \left(
    a_{kl_2}-a_{kl_1}
    -
    \lambda_h \Delta d
    \right),
    \label{eq:hop_ratio}
\end{equation}
where $\Delta d=|k-l_2|-|k-l_1|$.
If $l_2$ is farther from $k$ than $l_1$, then it must overcome an extra factor $\exp(-\lambda_h\Delta d)$ to receive larger attention.
\end{proposition}

\noindent\textit{Proof sketch.}
The same-row softmax normalizer cancels, leaving the semantic-score difference minus the hop-distance penalty in Eq.~\eqref{eq:hop_ratio}.
The proposition makes $\lambda_h$ interpretable as a distance prior.
Increasing it favors hop-synchronous matching, while smaller values allow stronger cross-hop transfers when the semantic attention score warrants them.

\begin{proposition}[Orthogonal Decoupling Controls Modality Collapse]
\label{prop:collapse}
Assume $\|p_i^{(m)}\|_2\leq R_p$ for all nodes and modalities.
Under the orthogonal loss in Eq.~\eqref{eq:orth_loss}, the mean projection of private residuals onto shared directions is bounded by
\vspace{-0.1em}
\begin{equation}
    \frac{1}{NM}
    \sum_{i,m}
    \left\|
    \operatorname{proj}_{s_i}
    \left(p_i^{(m)}\right)
    \right\|_2
    \leq
    R_p\sqrt{\mathcal{L}_{\perp}}.
    \label{eq:projection_bound}
\end{equation}
\vspace{-0.1em}
An analogous bound holds for projecting $s_i$ onto the direction of $p_i^{(m)}$ when $\|s_i\|_2$ is bounded.
\end{proposition}

\noindent\textit{Proof sketch.}
Eq.~\eqref{eq:orth_loss} averages squared cosines.
Projection length is bounded by the private norm times the absolute cosine, and Cauchy--Schwarz gives Eq.~\eqref{eq:projection_bound}.
The bound clarifies the role of $\mathcal{L}_{\perp}$ in the output layer.
If private residuals align too closely with $s_i$, then $e_i^{(m)}$ can be dominated by duplicated shared information.
Reducing the projection term keeps modality-specific semantics for downstream tasks.

Together, the results justify CoMAG modules through stable context, residual multi-hop trajectories, distance-regularized hop matching, and orthogonal decoupling of $s_i$ and $p_i^{(m)}$.

\section{Experiments}

We conduct experiments to provide a comprehensive evaluation of CoMAG from five perspectives.
\textbf{Q1}. How does CoMAG perform compared with representative baselines across graph-level and modality-level tasks, as reported in Tables~\ref{tab:graph_task_performance} and~\ref{tab:modality_task_performance}?
\textbf{Q2}. Do the key designed modules contribute to the overall effectiveness of CoMAG?
\textbf{Q3}. How sensitive is CoMAG to important hyperparameter settings?
\textbf{Q4}. How robust is CoMAG compared with baselines under text, image, edge, and label noise?
\textbf{Q5}. What is the theoretical computational complexity of CoMAG relative to representative MAG baselines, as summarized in Table~\ref{tab:complexity}?

\subsection{Experimental Setup}

\textbf{Datasets.} We evaluate CoMAG on nine multimodal attributed graph datasets from the OpenMAG benchmark~\cite{mu2024openmag}, including Movies, Grocery, Toys, DY, KU, Bili\_dance, RedditS, Flickr30k, and SemArt.
As summarized in Table~\ref{tab:datasets}, these datasets cover e-commerce, social media, video recommendation, image networks, and art networks.
All datasets contain text and visual node modalities and follow the train/validation/test splits provided by OpenMAG.

\textbf{Baselines.} We compare CoMAG with graph-only and multimodal attributed graph baselines.
For feature-only learning, we include MLP.
For graph-only topology learning, we include GCN~\cite{kipf2017gcn} and GAT~\cite{velickovic2018gat}.
For representative MAG methods, we follow the model set analyzed in Table~\ref{tab:complexity}, using DMGC~\cite{cheng2023dmgc} and DGF~\cite{yang2022dgf} as graph-enhanced methods, and MMGCN~\cite{wei2019mmgcn}, MGAT~\cite{tao2022mgat}, and LGMRec~\cite{LGMRec} as multimodal-enhanced methods.
We also include UniGraph2~\cite{he2025unigraph2} as a recent unified multimodal graph baseline.
CoMAG is evaluated as the proposed sparse-context modality-topology co-alignment model.

\textbf{Evaluation metrics.} We use task-specific metrics throughout the experiments.
For node classification, accuracy and F1-score indicate whether the learned representation supports reliable class prediction.
For link prediction, MRR and Hits@3 measure whether true edges are ranked near the top among candidate relations.
For node clustering, NMI and ARI reflect how well the discovered groups agree with ground-truth categories.
For modality matching, AUC and AP measure the separation between matched and mismatched cross-modal pairs.
For G2Text, BLEU-4 and CIDEr evaluate the lexical and consensus quality of graph-conditioned text generation.
For G2Image, CLIP-S and DINO-S indicate whether generated images preserve semantic alignment and visual consistency with the graph-conditioned target.

\textbf{Implementation details.}
Unless otherwise specified, all trainable models use a hidden dimension of 256, dropout rate of 0.3. We follow the official hyperparameter settings, and if such guidance is missing, we optimize the baseline performance with Adam Optimizer.

\textbf{Experiment Environment.}
The experiments are conducted on a machine with an AMD EPYC 7J13 64-Core Processor, and NVIDIA GeForce RTX 4090 with 24GB memory and CUDA 12.6. The operating system is Ubuntu 22.04.5 LTS with 503GB memory.

\begin{table*}[t]
\centering
\caption{Graph-level downstream task performance comparison. The best result is $\textbf{bold}$ and the second best result is $\underline{underlined}$.}
\label{tab:graph_task_performance}
\scriptsize
\setlength{\tabcolsep}{1.2pt}
\renewcommand{\arraystretch}{1.14}
\resizebox{\textwidth}{!}{
\begin{tabular}{l | c c c c | c c c c | c c c c}
\toprule
\multirow{3}{*}{\textbf{Methods}} &
\multicolumn{4}{c|}{\textbf{Node Classification}} &
\multicolumn{4}{c|}{\textbf{Link Prediction}} &
\multicolumn{4}{c}{\textbf{Node Clustering}} \\
\cmidrule{2-13}
& \multicolumn{2}{c}{Movies}
& \multicolumn{2}{c|}{Grocery}
& \multicolumn{2}{c}{DY}
& \multicolumn{2}{c|}{Bili\_Dance}
& \multicolumn{2}{c}{Toys}
& \multicolumn{2}{c}{RedditS} \\
& Acc & F1-score
& Acc & F1-score
& MRR & Hits@3
& MRR & Hits@3
& NMI & ARI
& NMI & ARI \\
\midrule

MLP
& $51.22_{\scriptsize \pm 0.12}$ & $0.40_{\scriptsize \pm 0.01}$
& $82.60_{\scriptsize \pm 0.26}$ & $0.75_{\scriptsize \pm 0.01}$
& $62.57_{\scriptsize \pm 0.89}$ & $75.65_{\scriptsize \pm 0.92}$
& $28.90_{\scriptsize \pm 0.27}$ & $34.87_{\scriptsize \pm 0.41}$
& $45.47_{\scriptsize \pm 0.68}$ & $29.20_{\scriptsize \pm 0.95}$
& $77.78_{\scriptsize \pm 0.15}$ & $69.90_{\scriptsize \pm 0.97}$ \\

GCN
& $52.51_{\scriptsize \pm 0.62}$ & $0.46_{\scriptsize \pm 0.01}$
& $77.11_{\scriptsize \pm 0.93}$ &
 $0.64_{\scriptsize \pm 0.01}$
& $70.28_{\scriptsize \pm 0.15}$ & $84.21_{\scriptsize \pm 0.15}$
& $37.79_{\scriptsize \pm 0.29}$ & $47.67_{\scriptsize \pm 0.27}$
& $48.03_{\scriptsize \pm 0.71}$ & $32.31_{\scriptsize \pm 0.95}$
& $78.39_{\scriptsize \pm 0.52}$ & $71.21_{\scriptsize \pm 0.96}$ \\

GAT
& $51.38_{\scriptsize \pm 0.64}$ & $0.43_{\scriptsize \pm 0.01}$
& $80.26_{\scriptsize \pm 0.32}$ & $0.73_{\scriptsize \pm 0.01}$
& $70.67_{\scriptsize \pm 0.23}$ & $84.96_{\scriptsize \pm 0.28}$
& $36.85_{\scriptsize \pm 0.36}$ & $47.36_{\scriptsize \pm 0.81}$
& $48.66_{\scriptsize \pm 0.23}$ & $31.96_{\scriptsize \pm 0.94}$
& $78.40_{\scriptsize \pm 0.90}$ & $68.90_{\scriptsize \pm 0.98}$ \\

\midrule

DMGC
& $52.35_{\scriptsize \pm 0.96}$ & $0.45_{\scriptsize \pm 0.03}$
& $80.06_{\scriptsize \pm 5.69}$ &
 $0.69_{\scriptsize \pm 0.08}$
& $74.47_{\scriptsize \pm 0.84}$ & $89.10_{\scriptsize \pm 0.73}$
& $39.95_{\scriptsize \pm 0.85}$ & $53.86_{\scriptsize \pm 0.92}$
& $49.38_{\scriptsize \pm 0.86}$ & $32.51_{\scriptsize \pm 0.88}$
& $79.72_{\scriptsize \pm 0.91}$ & $66.08_{\scriptsize \pm 0.95}$ \\

DGF
& $52.11_{\scriptsize \pm 0.64}$ & $0.37_{\scriptsize \pm 0.02}$
& $\underline{84.78}_{\scriptsize \pm 0.58}$ & $0.76_{\scriptsize \pm 0.02}$
& $\underline{77.28}_{\scriptsize \pm 0.17}$ & $\underline{92.51}_{\scriptsize \pm 0.18}$
& $\underline{42.55}_{\scriptsize \pm 0.12}$ & $\underline{58.25}_{\scriptsize \pm 0.37}$
& $\underline{51.29}_{\scriptsize \pm 0.62}$ & $\underline{36.24}_{\scriptsize \pm 0.91}$
& $\underline{84.89}_{\scriptsize \pm 0.36}$ & $\underline{78.07}_{\scriptsize \pm 0.74}$ \\

\midrule

MMGCN
& $53.04_{\scriptsize \pm 0.73}$ & $0.43_{\scriptsize \pm 0.03}$
& $83.53_{\scriptsize \pm 1.18}$ &
 $0.74_{\scriptsize \pm 0.02}$
& $69.66_{\scriptsize \pm 0.42}$ & $84.17_{\scriptsize \pm 0.55}$
& $37.70_{\scriptsize \pm 0.68}$ & $49.81_{\scriptsize \pm 0.88}$
& $45.39_{\scriptsize \pm 0.93}$ & $27.56_{\scriptsize \pm 0.48}$
& $67.90_{\scriptsize \pm 0.96}$ & $50.42_{\scriptsize \pm 0.99}$ \\

MGAT
& $52.39_{\scriptsize \pm 1.11}$ &
 $0.38_{\scriptsize \pm 0.06}$
& $82.61_{\scriptsize \pm 0.24}$ & $0.76_{\scriptsize \pm 0.01}$
& $70.37_{\scriptsize \pm 0.17}$ & $85.05_{\scriptsize \pm 0.30}$
& $36.74_{\scriptsize \pm 0.16}$ & $49.37_{\scriptsize \pm 0.69}$
& $46.72_{\scriptsize \pm 0.90}$ & $30.02_{\scriptsize \pm 0.93}$
& $73.36_{\scriptsize \pm 0.95}$ & $60.55_{\scriptsize \pm 0.98}$ \\

LGMRec
& $\underline{53.74}_{\scriptsize \pm 0.65}$ & $\underline{0.47}_{\scriptsize \pm 0.01}$
& $\underline{83.83}_{\scriptsize \pm 0.81}$ &
 $\underline{0.77}_{\scriptsize \pm 0.01}$
& $69.57_{\scriptsize \pm 0.04}$ & $85.08_{\scriptsize \pm 0.17}$
& $39.92_{\scriptsize \pm 0.40}$ & $51.85_{\scriptsize \pm 0.88}$
& $\underline{49.94}_{\scriptsize \pm 0.94}$ & $\underline{35.31}_{\scriptsize \pm 0.96}$
& $\underline{80.73}_{\scriptsize \pm 0.96}$ & $\underline{73.10}_{\scriptsize \pm 0.98}$ \\

\midrule

UniGraph2
& $46.78_{\scriptsize \pm 0.36}$ & $0.31_{\scriptsize \pm 0.02}$
& $75.86_{\scriptsize \pm 0.90}$ & $0.64_{\scriptsize \pm 0.03}$
& $65.78_{\scriptsize \pm 0.58}$ & $78.43_{\scriptsize \pm 0.64}$
& $31.58_{\scriptsize \pm 0.90}$ & $41.32_{\scriptsize \pm 0.95}$
& $10.34_{\scriptsize \pm 0.98}$ & $3.09_{\scriptsize \pm 0.94}$
& $31.85_{\scriptsize \pm 0.97}$ & $12.85_{\scriptsize \pm 0.92}$ \\

\midrule

\rowcolor[gray]{0.9} \textbf{CoMAG (Ours)}
& $\textbf{55.62}_{\scriptsize \pm 0.72}$ & $\textbf{0.49}_{\scriptsize \pm 0.03}$
& $\textbf{87.75}_{\scriptsize \pm 0.68}$ &
 $\textbf{0.80}_{\scriptsize \pm 0.03}$
& $\textbf{79.98}_{\scriptsize \pm 0.32}$ & $\textbf{95.75}_{\scriptsize \pm 0.28}$
& $\textbf{44.04}_{\scriptsize \pm 0.41}$ & $\textbf{60.29}_{\scriptsize \pm 0.63}$
& $\textbf{53.09}_{\scriptsize \pm 0.57}$ & $\textbf{37.51}_{\scriptsize \pm 0.72}$
& $\textbf{87.86}_{\scriptsize \pm 0.64}$ & $\textbf{80.80}_{\scriptsize \pm 0.79}$ \\

\bottomrule
\end{tabular}
}
\end{table*}
\begin{table}[H]
\centering
\caption{Performance Comparison on Modality-level tasks.}
\label{tab:modality_task_performance}
\scriptsize
\setlength{\tabcolsep}{1.2pt}
\renewcommand{\arraystretch}{1.14}
\resizebox{\columnwidth}{!}{
\begin{tabular}{l | c c | c c | c c}
\toprule
\multirow{3}{*}{\textbf{Methods}} &
\multicolumn{2}{c|}{\textbf{Modality Match}} &
\multicolumn{2}{c|}{\textbf{G2Text}} &
\multicolumn{2}{c}{\textbf{G2Image}} \\
\cmidrule{2-7}
& \multicolumn{2}{c|}{KU}
& \multicolumn{2}{c|}{Flickr30k}
& \multicolumn{2}{c}{SemArt} \\
& AUC & AP
& BLEU-4 & CIDEr
& CLIP-S & DINO-S \\
\midrule

MLP
& $86.40_{\scriptsize \pm 0.42}$ & $84.72_{\scriptsize \pm 0.48}$
& $5.87_{\scriptsize \pm 0.18}$ & $39.37_{\scriptsize \pm 0.64}$
& $67.54_{\scriptsize \pm 0.32}$ & $49.94_{\scriptsize \pm 0.45}$ \\

GCN
& $87.35_{\scriptsize \pm 0.37}$ & $85.90_{\scriptsize \pm 0.41}$
& $5.69_{\scriptsize \pm 0.21}$ & $38.44_{\scriptsize \pm 0.58}$
& $67.15_{\scriptsize \pm 0.29}$ & $49.65_{\scriptsize \pm 0.38}$ \\

GAT
& $87.80_{\scriptsize \pm 0.39}$ & $86.14_{\scriptsize \pm 0.44}$
& $5.74_{\scriptsize \pm 0.22}$ & $38.62_{\scriptsize \pm 0.61}$
& $67.28_{\scriptsize \pm 0.31}$ & $49.78_{\scriptsize \pm 0.40}$ \\

\midrule

DMGC
& $89.46_{\scriptsize \pm 0.33}$ & $88.10_{\scriptsize \pm 0.36}$
& $6.12_{\scriptsize \pm 0.19}$ & $41.68_{\scriptsize \pm 0.57}$
& $68.02_{\scriptsize \pm 0.28}$ & $51.44_{\scriptsize \pm 0.39}$ \\

DGF
& $90.18_{\scriptsize \pm 0.28}$ & $88.92_{\scriptsize \pm 0.33}$
& $\underline{6.83}_{\scriptsize \pm 0.17}$ & $\underline{44.28}_{\scriptsize \pm 0.52}$
& $68.43_{\scriptsize \pm 0.25}$ & $52.30_{\scriptsize \pm 0.34}$ \\

\midrule

MMGCN
& $89.02_{\scriptsize \pm 0.36}$ & $87.64_{\scriptsize \pm 0.42}$
& $5.82_{\scriptsize \pm 0.25}$ & $38.95_{\scriptsize \pm 0.62}$
& $67.62_{\scriptsize \pm 0.32}$ & $50.20_{\scriptsize \pm 0.46}$ \\

MGAT
& $88.76_{\scriptsize \pm 0.34}$ & $87.20_{\scriptsize \pm 0.39}$
& $5.79_{\scriptsize \pm 0.24}$ & $38.88_{\scriptsize \pm 0.60}$
& $67.78_{\scriptsize \pm 0.31}$ & $50.72_{\scriptsize \pm 0.44}$ \\

LGMRec
& $\underline{90.64}_{\scriptsize \pm 0.27}$ & $\underline{89.18}_{\scriptsize \pm 0.31}$
& $5.95_{\scriptsize \pm 0.20}$ & $39.00_{\scriptsize \pm 0.54}$
& $\underline{68.47}_{\scriptsize \pm 0.24}$ & $\underline{52.73}_{\scriptsize \pm 0.33}$ \\

\midrule

UniGraph2
& $84.05_{\scriptsize \pm 0.58}$ & $81.72_{\scriptsize \pm 0.65}$
& $5.28_{\scriptsize \pm 0.31}$ & $35.74_{\scriptsize \pm 0.72}$
& $65.86_{\scriptsize \pm 0.44}$ & $48.62_{\scriptsize \pm 0.57}$ \\

\midrule

\rowcolor[gray]{0.9} \textbf{CoMAG (Ours)}
& $\textbf{93.81}_{\scriptsize \pm 0.25}$ & $\textbf{92.30}_{\scriptsize \pm 0.29}$
& $\textbf{7.07}_{\scriptsize \pm 0.18}$ & $\textbf{45.83}_{\scriptsize \pm 0.49}$
& $\textbf{70.87}_{\scriptsize \pm 0.27}$ & $\textbf{54.58}_{\scriptsize \pm 0.36}$ \\

\bottomrule
\end{tabular}
}
\end{table}
\subsection{Performance Comparison}

To answer \textbf{Q1}, Tables~\ref{tab:graph_task_performance} and~\ref{tab:modality_task_performance} compare CoMAG with representative baselines on graph-level and modality-level tasks.

\noindent\textbf{Graph-level performance.}
Table~\ref{tab:graph_task_performance} shows that CoMAG achieves the strongest graph-level performance among the compared baselines across classification, link prediction, and clustering settings.
On Grocery, CoMAG reaches $87.75$, outperforming the strongest non-CoMAG result $84.78$.
On DY, it reaches $95.75$ compared with $92.51$ from DGF.
This advantage is not shared by all baselines.
Feature-only and graph-only models perform poorly when either semantic evidence or reliable neighborhood selection is required, while several multimodal baselines remain competitive on individual datasets but fall behind when structural prediction and clustering are considered together.
The observed behavior is consistent with CoMAG's reliable context design, where edge filtering suppresses noisy topology, semantic neighbors supplement missing relations, and the task-adaptive context keeps propagation aligned with the downstream objective.

\noindent\textbf{Modality-level performance.}
Table~\ref{tab:modality_task_performance} further shows that CoMAG remains the best among compared baselines on modality-level tasks.
The reported tasks emphasize cross-modal matching and generation quality, where fine-grained multimodal evidence creates clearer separation among methods.
On KU, CoMAG reports $93.81$ while LGMRec gives $90.64$.
On Flickr30k, CoMAG improves the graph-conditioned text generation score from DGF's $44.28$ to $45.83$.
Methods that rely mainly on graph aggregation or direct multimodal fusion can still perform reasonably, but they show weaker modality preservation when fine-grained cross-modal evidence is needed.
CoMAG's hop-token alignment and shared-private decoupling explain this pattern by aligning contextual evidence across modalities while keeping modality-specific cues available for matching and generation.

\subsection{Ablation Study}
\label{sec:ablation}

To answer \textbf{Q2}, Table~\ref{tab:ablation} compares the complete CoMAG with three variants that remove one key module at a time. The full model remains strongest across both the graph-level and generation settings, showing that the gains do not come from a single isolated design choice. Removing edge reliability weakens graph-level prediction, with the Grocery result falling from $87.75$ to $85.82$, because noisy observed edges are no longer filtered before propagation. Removing the semantic graph hurts both tasks by discarding missing semantic neighbors, which limits the context available beyond the raw topology. Removing the private residual causes the clearest degradation in G2Text quality, where the result falls to $40.86$, indicating that shared consensus alone is insufficient for preserving modality-specific generation evidence. Overall, the ablation study confirms that reliable topology, semantic context recovery, and modality-private evidence are complementary components of CoMAG.

\subsection{Hyperparameter Analysis}

The auxiliary losses $\lambda_\perp$ (orthogonal constraint) and $\lambda_s$ (graph smoothing) control the balance between graph-centric and modality-centric quality.
A higher $\lambda_\perp$ pushes the shared and private subspaces toward orthogonality, while $\lambda_s$ controls how strongly shared representations follow the learned context graph.
We explore $\lambda_\perp \in \{0.000, 0.005, 0.010, 0.015, 0.020\}$ and $\lambda_s \in \{0.00, 0.05, 0.10, 0.15, 0.20\}$.
\begin{table}[t]
\centering
\caption{Ablation studies on key modules of CoMAG.}
\label{tab:ablation}
\scriptsize
\setlength{\tabcolsep}{2.2pt}
\renewcommand{\arraystretch}{1.1}
\begin{tabular*}{\columnwidth}{@{\extracolsep{\fill}}lcccc@{}}
\toprule
\multirow{2}{*}{\textbf{Methods}} &
\multicolumn{2}{c}{\textbf{Grocery (Node Cls.)}} &
\multicolumn{2}{c}{\textbf{Flickr30k(G2Text)}} \\
\cmidrule(lr){2-3}\cmidrule(lr){4-5}
& Acc & F1 & BLEU4 & CIDEr \\
\midrule
Full CoMAG
& $87.75_{\scriptsize \pm 0.68}$
& $0.800_{\scriptsize \pm 0.03}$
& $7.07_{\scriptsize \pm 0.18}$
& $45.83_{\scriptsize \pm 0.49}$ \\
w/o Edge Reliability
& $85.82_{\scriptsize \pm 0.91}$
& $0.776_{\scriptsize \pm 0.04}$
& $6.72_{\scriptsize \pm 0.24}$
& $43.96_{\scriptsize \pm 0.62}$ \\
w/o Semantic Graph
& $86.21_{\scriptsize \pm 0.83}$
& $0.783_{\scriptsize \pm 0.03}$
& $6.48_{\scriptsize \pm 0.27}$
& $42.71_{\scriptsize \pm 0.71}$ \\
w/o Private Residual
& $86.90_{\scriptsize \pm 0.76}$
& $0.791_{\scriptsize \pm 0.03}$
& $6.11_{\scriptsize \pm 0.31}$
& $40.86_{\scriptsize \pm 0.84}$ \\
\bottomrule
\end{tabular*}
\end{table}
\begin{figure}[t]
\centering
\includegraphics[width=\columnwidth]{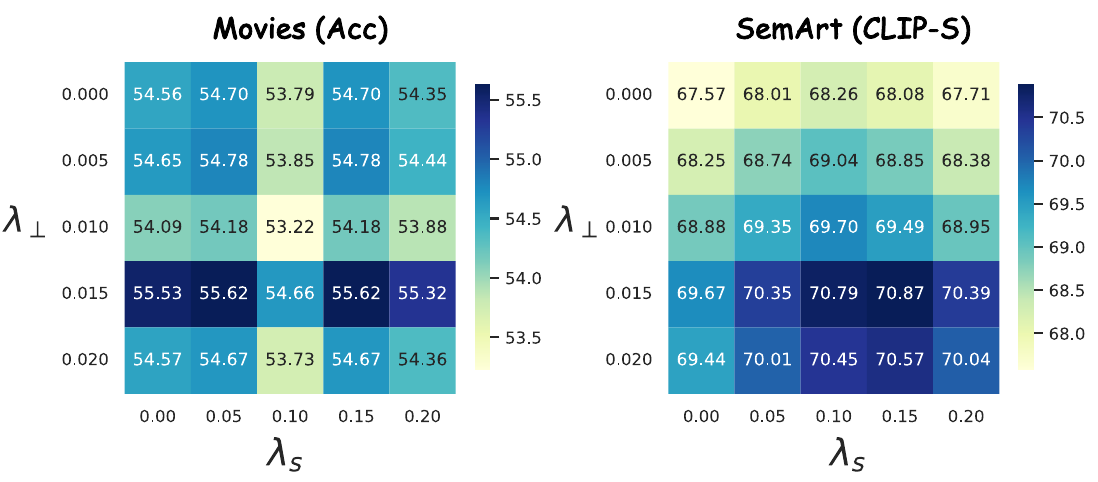}
\vspace{-0.5cm}
\caption{Hyperparameter sensitivity of two key parameters.}
\label{fig:hyperparameter}
\end{figure}

\begin{figure*}[t]
\centering
\includegraphics[width=\textwidth]{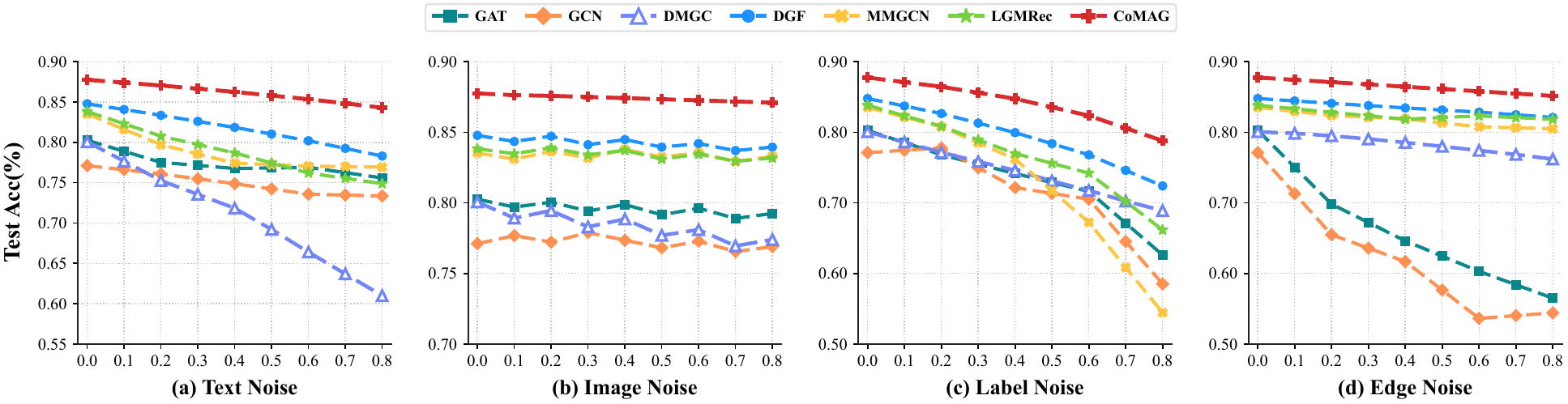}
\vspace{-0.5cm}
\caption{Robustness comparison under text, image, label, and edge noise.}
\label{fig:robustness}
\end{figure*}

To answer \textbf{Q3}, Fig.~\ref{fig:hyperparameter} shows that CoMAG is stable over a broad region rather than depending on one fragile setting.
Movies classification is strongest around $\lambda_\perp=0.015$, where moderate shared-private separation improves discriminative structure without suppressing shared graph evidence.
SemArt favors moderate-to-strong separation with $\lambda_s$ near $0.10$--$0.15$, indicating that smoothing helps image generation when it supports the learned context without over-constraining the shared representation.
This pattern is consistent with CoMAG's design, where orthogonal separation preserves private modality evidence and context smoothness is useful only when it remains aligned with reliable semantic neighborhoods.

\subsection{Robustness Analysis}

To answer \textbf{Q4}, Fig.~\ref{fig:robustness} compares CoMAG with representative baselines under text, image, label, and edge noise.
CoMAG keeps the strongest curve across all four settings, showing that its performance does not rely on a single clean source of evidence.
Label noise is the hardest setting because corrupted supervision directly changes the target signal, while image noise is comparatively mild because text and graph context can still stabilize the representation.
Edge noise especially separates CoMAG from topology-dependent baselines, as reliable edge filtering and semantic context recovery reduce the impact of corrupted observed topology.
The robustness trends also support the shared-private design, since modality-private residuals help preserve useful evidence when one modality or structural channel becomes unreliable.

\subsection{Theoretical Complexity Analysis}

We analyze the backbone-level complexity of representative MAG methods in Table~\ref{tab:complexity}.
Let $|\mathcal{V}|$ and $|\mathcal{E}|$ denote the number of nodes and edges, $M$ the number of modalities, $L$ the number of baseline GNN layers, $d$ the hidden dimension, $K$ a baseline-specific cluster or iteration count, and $|Q_\tau|$ the number of nonzero edges in CoMAG's learned sparse context graph.
For CoMAG, the hop depth $K_h$ is small and fixed, so its hop-token alignment overhead is absorbed into constant factors.
\begin{table}[t]
\centering
\caption{Theoretical complexity analysis between baselines.}
\label{tab:complexity}
\scriptsize
\setlength{\tabcolsep}{1.6pt}
\renewcommand{\arraystretch}{1.05}
\resizebox{\columnwidth}{!}{%
\begin{tabular}{l c c c}
\toprule
\textbf{Models} & \textbf{Training} & \textbf{Inference} & \textbf{Memory} \\
\midrule
DMGC & $O(|\mathcal{V}|^2 d + K|\mathcal{V}|d)$ & $O(|\mathcal{V}|^2 d)$ & $O(|\mathcal{V}|^2 + Kd)$ \\
DGF & $O(|\mathcal{E}|d + |\mathcal{V}|d^2)$ & $O(|\mathcal{E}|d)$ & $O(|\mathcal{V}|d + |\mathcal{E}|)$ \\
\midrule
MMGCN & $O(ML|\mathcal{E}|d)$ & $O(ML|\mathcal{E}|d)$ & $O(ML|\mathcal{V}|d)$ \\
MGAT & $O(ML(|\mathcal{V}|d^2+|\mathcal{E}|d))$ & $O(ML(|\mathcal{V}|d^2+|\mathcal{E}|d))$ & $O(ML|\mathcal{V}|d + |\mathcal{E}|)$ \\
LGMRec & $O(L|\mathcal{E}|d + M|\mathcal{V}|d)$ & $O(L|\mathcal{E}|d + M|\mathcal{V}|d)$ & $O((M+L)|\mathcal{V}|d + |\mathcal{E}|)$ \\
\midrule
\textbf{CoMAG} & $O(M(|\mathcal{E}|+|Q_\tau|)d)$ & $O(M(|\mathcal{E}|+|Q_\tau|)d)$ & $O(M|\mathcal{V}|d + |Q_\tau|)$ \\
\bottomrule
\end{tabular}
}
\end{table}
$Q_\tau$ denotes the learned sparse context graph.
Since $M$ and $K_h$ are small in practice and $|Q_\tau|$ is controlled by sparse semantic neighbor selection, CoMAG has edge-linear complexity comparable to sparse multimodal GNNs while avoiding the dense $O(|\mathcal{V}|^2)$ memory cost of structure-learning methods.

\section{Conclusion}

We introduced CoMAG, a unified framework for multimodal attributed graph learning that simultaneously serves graph-centric and modality-centric tasks from a single forward pass.
Our experiments compare CoMAG with representative feature-only, graph-only, multimodal, and unified MAG baselines across graph-level and modality-level OpenMAG tasks.
Through four integrated modules (edge reliability context graph, modality-specific hop trajectories, hop-token cross-modal matching, and shared-private decoupling), CoMAG provides task-adaptive context construction and modality-preserving alignment with formal theoretical guarantees covering linear convergence to a unique fixed point, provable over-smoothing mitigation, and bounded modality collapse.
Experiments on both modality and graph-level tasks confirm CoMAG's consistent advantages over baselines.
Future work will pursue broader domain evaluations, and joint multi-task training with a calibrated loss schedule to further strengthen cross-task generalization.

\bibliographystyle{IEEEtran}
\bibliography{references}

\end{document}